\documentclass[11pt]{article}
\usepackage{mathrsfs}
\usepackage{amsmath}
\usepackage{amsfonts}
\usepackage{amssymb}
\usepackage{latexsym}
\usepackage{enumerate}
\usepackage{graphicx}
\usepackage{color}
\usepackage{boxedminipage}
\newtheorem{theorem}{\bf Theorem}[section]
\newtheorem{lemma}[theorem]{\bf Lemma}

\newtheorem{corollary}[theorem]{\bf Corollary}
\newtheorem{example}[theorem]{\bf Example}
\newtheorem{definition}[theorem]{\bf Definition}

\newenvironment{proof}{\noindent{\em Proof:}}{\quad \hfill$\Box$\vspace{2ex}}

\usepackage{hyperref}
\usepackage{pdflscape}
\usepackage{ulem}

\def \bR {\mathbb R}
\def \bN {\mathbb N}
\def \bZ {\mathbb Z}
\def \bC {\mathbb C}
\def \bP {\mathbb P}
\def \bx {{\bf x}}

\def \bz {{\bf z}}

\def \bt {{\bf t}}

\def \cB {{\cal B}}
\def \cL {{\cal L}}

\def \cH {{\cal H}}
\def \cE {{\cal E}}
\def \cF {{\cal F}}
\def \cN {{\cal N}}
\def \cW {{\cal W}}

\def \cM {{\cal M}}
\def \cG {{\cal G}}
\def \and {\, \mbox{\rm and}\, }

\def \span {\,{\rm span}\,}
\def \supp {\,{\rm supp}\,}

\def \sgn {\,{\rm sgn}\,}

\begin{document}
\title{\bf On Reproducing Kernel Banach Spaces: Generic Definitions and Unified Framework of Constructions}
\author{Rongrong Lin\thanks{School of Mathematics and Statistics, Guangdong University of Technology, Guangzhou 510520, P. R. China. E-mail address: {\it linrr@gdut.edu.cn}.
Supported in part by Natural Science Foundation of China under grant 11901595.},
\quad Haizhang Zhang\thanks{School of Mathematics (Zhuhai), and Guangdong Province Key Laboratory of Computational Science, Sun Yat-sen University, Zhuhai 519082, P. R. China. E-mail address: {\it zhhaizh2@mail.sysu.edu.cn}. Supported in part by Natural Science Foundation of China under grant 11971490, and by Natural Science Foundation of Guangdong Province under grant 2018A030313841.}, \quad and\quad Jun Zhang\thanks{Department of Psychology and Department of Mathematics, University of Michigan, Ann Arbor, MI 48109, USA. E-mail address: {\it junz@umich.edu}. Supported in part by DARPA/ARO under grant \#W911NF-16-1-0383.}}

\date{}
\maketitle
\begin{abstract} Recently, there has been emerging interest in constructing reproducing kernel Banach spaces (RKBS) for applied and theoretical purposes such as machine learning, sampling reconstruction, sparse approximation and functional analysis. Existing constructions include the reflexive RKBS via a bilinear form, the semi-inner-product RKBS, the RKBS with $\ell^1$ norm, the $p$-norm RKBS via generalized Mercer kernels, etc. The definitions of RKBS and the associated reproducing kernel in those references are dependent on the construction. Moreover, relations among those constructions are unclear. We explore a generic definition of RKBS and the reproducing kernel for RKBS that is independent of construction. Furthermore, we propose a framework of constructing RKBSs that leads to new RKBSs based on Orlicz spaces and unifies existing constructions mentioned above via a continuous bilinear form and a pair of feature maps.
Finally, we develop representer theorems for machine learning in RKBSs constructed in our framework, which also unifies representer theorems in existing RKBSs.

\end{abstract}

\section{Introduction}
In this paper, we  aim at the construction of reproducing kernel Banach spaces (RKBSs), which serve as a generalization of reproducing kernel Hilbert spaces (RKHSs) \cite{SS02}. The notion of RKBS was originally introduced in machine learning in 2009 \cite{ZXZ09}. Since then, various RKBSs \cite{FHY15,GSP14,SZ11,SZH13,XY,Ye2013,ZXZ09,ZZ12} have been constructed for different applied and theoretical purposes. They are used in a wide variety of fields such as machine learning \cite{XY,Ye2013,ZXZ09,ZZ12,ZZ13}, sampling reconstruction \cite{GA13,GHM14,HNS09,NS10,NSX13}, sparse approximation \cite{SZ11,SZH13,XY,Ye2013}, and functional analysis \cite{DJ16,SFL11,ZZ11}.
The definitions of existing RKBSs and the associated reproducing kernels in the literature are dependent on the construction.
To address this issue and to further promote research in the subject, it is helpful to understand the essence of RKBS and to propose a framework of constructing RKBSs that unifies existing constructions.

RKHSs are Hilbert spaces of functions on which point evaluation functionals are continuous \cite{SS02}. In machine learning, RKHSs have been viewed as ideal spaces for kernel-based learning algorithms. Thanks to the existence of an inner product, Hilbert spaces  are well-understood in functional analysis. Many applications require that the sampling process to be stable. In other words, the point evaluations should be bounded. Most importantly, an RKHS has a reproducing kernel, which measures similarity between input data and gives birth to the ``kernel trick" in machine learning that significantly saves computations. Successful and important machine learning methods based on RKHSs include support vector machines and the regularization networks \cite{SS02,SC08}.

There are many reasons that justify the need of RKBSs. We mention three here. First of all, Banach spaces possess richer geometrical structures and norms.  It is well-known that any two Hilbert spaces over $\bC$ of the same dimension are isometrically isomorphic. By contrast,  for $1\le p\ne q\le+\infty$, $L^p([0,1])$ and $L^q([0,1])$ are not  isomorphic to each other (see, \cite{Fabian01}, page 180). Secondly, kernel functions play the role of measuring similarity of inputs in machine learning. They are defined by inner product through a feature map and therefore are inherently symmetric. In some applications such as psychology \cite{ZZBook}, asymmetric kernels are desired, which can only be obtained via Banach spaces. Thirdly, machine learning schemes in Banach spaces have received considerable attention recently \cite{AMP10,DL07, MP04,MP07,SZ11,SZH13, XY,Ye2013, ZXZ09, ZZ12, ZZ13}.   Many important problems such as $p$-norm coefficient-based regularization \cite{SFZ11, SZH13,TCY10,XZ10},  large-margin classification \cite{DL07,Ye2013,ZXZ09}, lasso in statistics \cite{Tibshirani96} and compressed sensing \cite{CRT06} had better be studied in Banach spaces. It suggests the need of extending Hilbert space type arguments to Banach spaces.

Under these considerations, different definitions and constructions of RKBSs have been proposed in the literature. In 2009, Zhang et al. \cite{ZXZ09} proposed the new concept of RKBS for machine learning. Reflexive RKBSs were constructed by the bilinear form between a reflexive Banach space and its dual. The semi-inner-product RKBS was also studied in \cite{ZXZ09} and \cite{ZZ11,ZZ12} based on the tool of semi-inner products \cite{Giles67,Lumer61}.  For the multi-task learning, Zhang et al. \cite{ZZ13} developed the notion of semi-inner-product vector-valued RKBS. In 2013, Song et al. \cite{SZH13} constructed a class of RKBSs with the $\ell^1$ norm via admissible kernels targeting at sparse learning. In those spaces, the representer theorem for regularized learning schemes is satisfied.  In 2014, Georgiev et al. \cite{GSP14} constructed a class of RKBSs with the $p$-norm ($1\le p\le +\infty$) without the representer theorem. In 2015, Fasshauer et al. \cite{FHY15} constructed a class of RKBSs with positive definite functions. More recently, the $p$-norm ($1\le p\le +\infty$) RKBSs were systematically developed by Xu and Ye in \cite{XY} via generalized Mercer kernels.

The definitions of RKBSs and the associated reproducing kernels in the above references are dependent on the construction. Moreover, the relation among those constructions is unclear. Limitations also persist. For instance, the Banach space $C([0,1])$ of all continuous functions on the interval $[0,1]$ does not satisfy those definitions. However, $C([0,1])$ should be an RKBS as point evaluations are clearly continuous in the space. We hope to propose a framework of constructing RKBSs that unifies existing constructions, and leads further to new constructions that accept $C([0,1])$ as a particular example. The main purpose of this paper is two-fold:
\begin{description}
\item[(i)] to give a generic definition of RKBS that naturally generalizes the classical RKHS and is independent of construction;
\item[(ii)] to propose a unified framework of constructing RKBSs that covers all existing constructions in the literature and also leads to new RKBSs.
\end{description}

The outline of the paper is as follows. In Section 2, we first give a generic definition of RKBSs, and then propose a novel framework of constructing RKBSs via a pair of feature maps. Furthermore, we are able to construct a new class of Orlicz RKBSs. In Section 3, we justify that our framework does unify existing RKBSs in the literature. In particular, $C([0,1])$ is included. In the last section, we develop a representer theorem for regularization networks in RKBSs constructed in our framework.

\section{Generic definitions and constructions}

In this section, we first present a generic definition of RKBS and the reproducing kernel for RKBS. Then, we propose a novel framework of constructing RKBSs via a pair of feature maps. Finally, we construct new RKBSs based on Orlicz spaces.

\subsection{Generic definitions of RKBSs}

We describe our generic definition of RKBSs as follows.

\begin{definition}\label{definition1} ({\bf Reproducing Kernel Banach Spaces (RKBS)}) A {\it reproducing kernel Banach space}  $\cB$ on a prescribed nonempty set $X$ is a Banach space of certain functions on $X$ such that every point evaluation functional $\delta_x$, $x\in X$ on $\cB$ is continuous, that is, there exists a positive constant $C_x$ such that
$$
|\delta_x (f)|=|f(x)|\le C_x\|f\|_{\cB}\mbox{ for all }f\in\cB.
$$
\end{definition}
Note that a normed vector space $V$ on $X$ is called a {\it Banach space of functions} if it is a Banach space whose elements are functions on $X$ and for each $f\in V$, $\|f\|_{V}=0$ if and only if $f$, as a function, vanishes everywhere on $X$. By definition, $L^p([0,1])$, $1\le p\le +\infty$, is  not a Banach space of functions as it consists of equivalent classes of functions with respect to the Lebesgue measure.

Definition \ref{definition1} naturally generalizes the classical definition of RKHS. We should point out that such a definition was implicitly mentioned in some papers on sampling theorems \cite{Christensen12,GA13,HNS09}, although no reproducing kernels were defined or even mentioned therein.  We also point out that the authors in \cite{ZXZ09} had also intended to use this definition for RKBSs, but eventually gave it up due to the example $C([0,1])$. A semi-inner-product structure was imposed in \cite{ZXZ09} to ensure the existence of a reproducing kernel.

By Definition \ref{definition1}, $C([0,1])$ is an RKBS. We shall see what its reproducing kernels look like in Subsection \ref{Subsection3.5}. The only requirement on RKBSs in our definition is continuity of point evaluations.
Definitions of RKBSs in existing literature \cite{SZH13,XY,ZXZ09,ZZ11} all impose other requirements to ensure the existence of a reproducing kernel that is not a generalized function. Those requirements more or less seem unnatural. We are able to remove them by exploiting the definition of reproducing kernels via continuous bilinear forms.

A bilinear form between two normed vector spaces $V_1,V_2$ is a function $\langle \cdot,\cdot\rangle_{V_1\times V_2}$ from $V_1\times V_2$ to $\bC$  that is linear about both arguments. It is said to be {\it continuous}  if  there exists a positive constant $C$ such that
$$
|\langle f,g\rangle_{V_1\times V_2}| \le C\|f\|_{V_1}\|g\|_{V_2}\mbox{ for all }f\in V_1,g\in V_2.
$$

\begin{definition}\label{definition2} ({\bf Reproducing Kernels for RKBS}) Let $\cB_1$ be an RKBS on a set $\Omega_1$. If there exists a Banach space $\cB_2$ of functions on another set $\Omega_2$, a continuous bilinear form  $\langle \cdot,\cdot\rangle_{\cB_1\times\cB_2}$, and a function $K$ on $\Omega_1\times\Omega_2$ such that $K(x,\cdot)\in\cB_2$ for all $x\in\Omega_1$ and
 \begin{equation}\label{Reproducingeq1}
f(x)=\langle f,K(x,\cdot)\rangle_{\cB_1\times\cB_2} \mbox{ for all }x\in\Omega_1\mbox{ and all }f\in\cB_1,
\end{equation}
then we call $K$ a {\it reproducing kernel} for $\cB_1$. If, in addition, $\cB_2$ is also an RKBS on $\Omega_2$ and it holds $K(\cdot,y)\in\cB_1$ for all $y\in\Omega_2$ and
 \begin{equation}\label{Reproducingeq2}
g(y)=\langle K(\cdot,y),g\rangle_{\cB_1\times\cB_2}\mbox{ for all }y\in\Omega_2\mbox{ and all }g\in\cB_2,
\end{equation}
then we call $\cB_2$ an {\it adjoint RKBS} of $\cB_1$ and call $\cB_1$ and $\cB_2$ {\it a pair of RKBSs}. In this case, $\widetilde{K}(x,y):=K(y,x)$, $x\in\Omega_2,y\in\Omega_1$, is a reproducing kernel for $\cB_2$.
\end{definition}

We call (\ref{Reproducingeq1}) and (\ref{Reproducingeq2})  the {\it reproducing properties} for the kernel $K$ in RKBSs $\cB_1$ and $\cB_2$, respectively.

Although there are many conditions in the definition, we shall see that RKBSs and reproducing kernels satisfying the conditions can be easily constructed via a pair of feature maps.

\subsection{Constructions via a pair of feature maps}

We shall propose a unified framework of constructing RKBSs via a pair of feature maps. We shall discuss in the next section that all existing constructions of RKBSs fall into the framework.

Let $\langle \cdot,\cdot\rangle_{\cW_1\times\cW_2}$ be a continuous bilinear form $\langle \cdot,\cdot\rangle_{\cW_1\times\cW_2}$ on Banach spaces $\cW_1$ and $\cW_2$. We call the linear span $\span A$ of a set $A\subseteq \cW_1$ dense in $\cW_1$ with respect to the bilinear form $\langle \cdot,\cdot\rangle_{\cW_1\times\cW_2}$ if for any $v\in\cW_2$,
$$
\langle a,v\rangle_{\cW_1\times\cW_2}=0\mbox{ for all }a\in A
$$
implies $v=0$. Similarly, we can define denseness in $\cW_2$ with respect to the bilinear form.

Our construction is described below.

\begin{center}
\bf Constructions of RKBSs
\end{center}
\begin{center}
\begin{boxedminipage}{15cm}
Let $\cW_1,\cW_2$ be two Banach spaces, and $\langle \cdot,\cdot\rangle_{\cW_1\times\cW_2}$ be a continuous bilinear form on $\cW_1\times\cW_2$.
Suppose there exist two nonempty sets $\Omega_1$ and $\Omega_2$, and mappings
$$
\Phi_1:\Omega_1\to\cW_1, \ \ \Phi_2:\Omega_2\to\cW_2
$$
such that with respect to the bilinear form
\begin{equation}\label{density*}
\span\Phi_1(\Omega_1) \mbox{ is dense in }\cW_1,\  \span\Phi_2(\Omega_2) \mbox{ is dense in }\cW_2.
\end{equation}
We construct
\begin{equation}\label{B1}
\cB_1:=\Big\{f_v(x):=\langle \Phi_1(x),v\rangle_{\cW_1\times\cW_2}: v\in\cW_2, x\in\Omega_1\Big\}
\end{equation}
with norm $$\|f_v\|_{\cB_1}:=\|v\|_{\cW_2}$$ and
\begin{equation}\label{B2}
\cB_2:=\Big\{g_u(y):=\langle u,\Phi_2(y)\rangle_{\cW_1\times\cW_2}: u\in\cW_1, y\in\Omega_2\Big\}
\end{equation}
with norm $$\|g_u\|_{\cB_2}:=\|u\|_{\cW_1}.$$
\end{boxedminipage}
\end{center}
$\\$

The construction can be simplified when $\cW_2$ is a closed subspace of $\cW_1^*$,  the dual space of continuous linear functionals on $\cW_1$. In this case, we always use the natural continuous bilinear form
$$
\langle u,v\rangle_{\cW_1\times\cW_2}=\langle u,v\rangle_{\cW_1}:=v(u),\ \ u\in\cW_1,\ v\in\cW_2.
$$
The denseness condition (\ref{density*}) is satisfied if
\begin{equation}\label{density}
\cW_1=\overline{\span}\Phi_1(\Omega_1),\  \cW_1^*=\overline{\span} \Phi_2(\Omega_2)
\end{equation}
or
\begin{equation}\label{densityweak}
\cW_1=\overline{\span}\Phi_1(\Omega_1),\mbox{ and }\span\Phi_2(\Omega_2)\mbox{ is dense in }\cW_1^*\mbox{ under the weak}^*\mbox{ topology}.
\end{equation}

We show that such constructed spaces $\cB_1$ and $\cB_2$ indeed form a pair of RKBSs, and present the associated reproducing kernel.

\begin{theorem}\label{Theorem1} Let $\cB_1$ and $\cB_2$ be constructed as in (\ref{B1}) and (\ref{B2}), respectively. Then with the bilinear form on $\cB_1\times\cB_2$
\begin{equation}\label{continuousform1}
\langle f_v,g_u\rangle_{\cB_1\times \cB_2}:= \langle u,v\rangle_{\cW_1\times\cW_2} \mbox{ for all }f_v\in\cB_1\mbox{ and all }g_u\in\cB_2,
\end{equation}
$\cB_1$ is an RKBS on $\Omega_1$ with the adjoint RKBS $\cB_2$ on $\Omega_2$. Moreover,
\begin{equation}\label{kernel}
K(x,y):=\langle \Phi_1(x),\Phi_2(y)\rangle_{\cW_1\times \cW_2},\ x\in\Omega_1,y\in\Omega_2,
\end{equation}
is a reproducing kernel for $\cB_1$.
\end{theorem}
\begin{proof}
First note that the denseness condition (\ref{density*}) guarantees that the $v,u$ in $f_v$ and $g_u$ are both unique. The definition is hence well-defined. We next prove that $\cB_1$ and $\cB_2$ are RKBSs. By assumption, there exists a positive constant $C$ such that
\begin{equation}\label{continuousform}
|\langle u,v\rangle_{\cW_1\times\cW_2}|\le C\|u\|_{\cW_1}\|v\|_{\cW_2}\mbox{ for all }u\in\cW_1\mbox{ and all }v\in\cW_2.
\end{equation}
By (\ref{B1}) and (\ref{continuousform}), we have  for all $f_v\in\cB_1$ and all $x\in\Omega_1$
$$
|f_v(x)|=|\langle \Phi_1(x),v\rangle_{\cW_1\times \cW_2}|\le C\|\Phi_1(x)\|_{\cW_1}\|v\|_{\cW_2}=C\|\Phi_1(x)\|_{\cW_1}\|f_v\|_{\cB_1}.
$$
Similarly, by (\ref{B2}) and (\ref{continuousform}), we have for all $g_u\in\cB_2$ and all $y\in\Omega_2$
$$
|g_u(y)|=|\langle u,\Phi_2(y)\rangle_{\cW_1\times\cW_2}|\le C\|\Phi_2(y)\|_{\cW_2}\|u\|_{\cW_1}=C\|\Phi_2(y)\|_{\cW_2}\|g_u\|_{\cB_2}.
$$
Thus, point evaluation functionals are continuous on both $\cB_1$ and $\cB_2$. By Definition \ref{definition1}, $\cB_1$ and $\cB_2$ are RKBSs on $\Omega_1$ and $\Omega_2$, respectively.

Next, $\langle \cdot,\cdot\rangle_{\cB_1\times \cB_2}$ is a continuous bilinear form as
$$
|\langle f_v,g_u\rangle_{\cB_1\times \cB_2}|=|\langle u,v\rangle_{\cW_1\times\cW_2}|\le C\|u\|_{\cW_1}\|v\|_{\cW_2} = C\|f_v\|_{\cB_1}\|g_u\|_{\cB_2}.
$$

Finally, by (\ref{density*}), (\ref{B1}), (\ref{continuousform1}) and (\ref{kernel}), we have
$$
K(x,\cdot)=g_{\Phi_1(x)}\in\cB_2,\ \ f_v(x)=\langle \Phi_1(x),v\rangle_{\cW_1\times\cW_2}=\langle f_v, g_{\Phi_1(x)}\rangle_{\cB_1\times \cB_2}=
\langle f_v, K(x,\cdot)\rangle_{\cB_1\times \cB_2}
$$
for all $x\in\Omega_1$ and all $f_v\in\cB_1$. Similarly,
$$
K(\cdot,y)=f_{\Phi_2(y)}\in\cB_1,\ \ g_u(y)=\langle u,\Phi_2(y)\rangle_{\cW_1\times\cW_2}=\langle f_{\Phi_2(y)},g_u\rangle_{\cB_1\times \cB_2}=
\langle K(\cdot,y),g_u\rangle_{\cB_1\times \cB_2}
$$
for all $y\in\Omega_2$ and all $g_u\in\cB_2$. The proof is complete.
\end{proof}

We call $\Phi_1:\Omega_1\to\cW_1$ and $\Phi_2:\Omega_2\to\cW_2$  in Theorem \ref{Theorem1}  {\it a pair of feature maps} for the reproducing kernel $K$, and $\cW_1$ and $\cW_2$ {\it a pair of feature spaces}  for $K$.  The feature spaces and feature maps for an RKBS might not be unique. Different choices of them will lead to different reproducing kernels for the RKBS.

Below we prove the existence of a reproducing kernel for a general RKBS under the mild condition that it is separable.

\begin{theorem}
Every separable RKBS admits a reproducing kernel.
\end{theorem}
\begin{proof}
Let $\cB$ be a separable RKBS on $X$. Thus, there exist $f_n\in\cB$, $n\in\bN$ such that $\span\{f_n:n\in\bN\}$ is dense in $\cB$. We construct a pair of feature maps and feature spaces for $\cB$ as follows. Let
$$
\cW_1=\overline{\span}\{\delta_x:x\in X\}\mbox{ in }\cB^*,\ \Omega_1=X,\ \ \mbox{and }\Phi_1(x)=\delta_x.
$$
Also, set
$$
\cW_2=\cB,\ \Omega_2=\bN,\ \ \mbox{and }\Phi_2(n)=f_n(x).
$$
Clearly, the denseness condition (\ref{density*}) is satisfied. We observe that the space $\cB_1$ followed from the construction is exactly $\cB$. By Theorem \ref{Theorem1},
$$
K(x,n)=\langle \delta_x,f_n\rangle=f_n(x),\ \ x\in \Omega_1,\ n\in\bN
$$
is a reproducing kernel for $\cB$.
\end{proof}

We remark that the reproducing kernel given in the above theorem is of little use in practice. The purpose of the theorem is to verify that Definition \ref{definition1} is indeed generic for RKBSs. Under the definition, reproducing kernels for RKBSs exist under a very mild condition, namely, separability.

\subsection{Orlicz RKBSs}

To show the advantages of our framework, we construct a new class of RKBSs based on Orlicz spaces. Orlicz spaces constitute an important generalization of $L^p$-spaces. To this end, we shall briefly review the basic theory of Orlicz spaces \cite{RR91,RR02}.

Set $\bR_+:=[0,+\infty)$.  Let $\varphi:\bR_+\to\bR_+$ be a strictly increasing and continuous function with $\varphi(0)=0$ and $\lim_{t\to +\infty}\varphi(t)=+\infty$.
Besides $\varphi$, the following three functions are central to our discussions:
$$
\psi(t):=\varphi^{-1}(t),\ \ \Phi(t):=\int_0^t \varphi(s)ds,\ \ \Psi(t):=\int_0^t \psi(s)ds,\ \ t\in\bR_+.
$$
We call $\Phi$ and $\Psi$ a pair of conjugated {\it nice Young functions}. Then
$$xy\le \Phi(x)+\Psi(y),\ \ x,y\in\bR_+,$$ where equality holds if $x=\psi(y)$, or equivalently $y=\varphi(x)$. The nice Young function $\Psi$ complementary to $\Phi$ can equally be defined by $\Psi(t)=\sup\{x|t|-\Phi(x):x\in\bR_+\}$, $t\in\bR$.

Let $(\Omega,\cF,\mu)$ be a measure space.  The Orlicz space \cite{RR02} $\cL^\Phi$  consists of all the $\cF$-measurable functions $f:\Omega\to\bC$ such that
$$
\int_\Omega\Phi(\alpha|f|)d\mu<+\infty\mbox{ for some }\alpha>0.
$$
It is a linear vector space with two equivalent norms $\|\cdot\|_\Phi$, the Gauge norm (Luxemburg norm), and $|\cdot|_\Phi$, the Orlicz norm. They are respectively defined by
$$
\|f\|_\Phi:=\inf\Big\{\alpha>0:\int_\Omega \Phi\Big(\frac{|f|}\alpha\Big)d\mu\le \Phi(1)\Big\},\ \ f\in \cL^\Phi
$$
and
$$
|f|_\Phi:=\sup\Big\{\int_\Omega|fg|d\mu:g\in \cL^\Psi,\ \|g\|_\Psi\le1\Big\},\ \ f\in \cL^\Phi.
$$
They indeed are equivalent as $\Phi(1)\|f\|_\Phi\le |f|_\Phi\le 2 \|f\|_\Phi$, $f\in\cL^\Phi$. If $\varphi(t)=t^{p-1}$ and $\psi(t)=t^{q-1}$, $t\in\bR_+$, where $1<p,q<+\infty$ with $1/p+1/q=1$, then
$$
\Phi(t)=t^p/p,\ \Psi(t)=t^q/q,\ t\in\bR_+.
$$
In this case, $\cL^\Phi$ is the usual space $L^p_\mu(\Omega)$ and both the Gauge norm and Orlicz norm equal
$$
\|f\|_{p}:=\Big(\int_\Omega|f(x)|^pd\mu(x)\Big)^{1/p},\ \ f\in L^p_\mu(\Omega).
$$

If $\Phi$ and $\Psi$ are a pair of conjugated nice Young functions then $(\cL^\Psi,|\cdot|_\Psi)=(\cL^\Phi,\|\cdot\|_\Phi)^*$ (see, Theorem 13 in Section 1.2, \cite{RR02}) in the sense that for each continuous linear functional $T$ on
$(\cL^\Phi,\|\cdot\|_\Phi)$ there exists a unique $g\in \cL^\Psi$ such that
$$
T(f)=\int_\Omega fgd\mu,\ f\in \cL^\Phi\quad\mbox{and}\quad
\|T\|=|g|_\Psi.
$$
There is a unique semi-inner product (\cite{Lumer63}, pp. 101--104) on $(\cL^\Phi,\|\cdot\|_\Phi)$ given by
$$
[f,g]=\|g\|_\Phi^2\frac{\displaystyle{\int_\Omega f\sgn(\bar{g})\varphi\Big(\frac{|g|}{\|g\|_\Phi}\Big)d\mu}}{\displaystyle{\int_\Omega |g|\varphi\Big(\frac{|g|}{\|g\|_\Phi}\Big)d\mu}},\ \ f,g\in \cL^\Phi,
$$
where $\sgn(t):=t/|t|$ if $t\in\bC\setminus\{0\}$ and $\sgn(0):=0$. It implies that the standard duality mapping $J_\Phi:(\cL^\Phi,\|\cdot\|_\Phi)\to (\cL^\Psi,|\cdot|_\Psi)$ has the form
$$
J_\Phi(f)=\frac{\displaystyle{\|f\|_\Phi^2\sgn(\bar{f})\varphi\Big(\frac{|f|}{\|f\|_\Phi}\Big)}}{\displaystyle{\int_\Omega |f|\varphi\Big(\frac{|f|}{\|f\|_\Phi}\Big)d\mu}},\ \ f\in \cL^\Phi.
$$
Particularly, when $\cL^\Phi=L^p_\mu(\Omega)$, we observe that
$$
J_p(f)=\frac{\sgn(\bar{f})|f|^{p-1}}{\|f\|_p^{p-2}},\ \ f\in L^p_\mu(\Omega),
$$
which is the usual duality mapping from $L^p_\mu(\Omega)$ to $L^q_\mu(\Omega)$ first discovered by Giles \cite{Giles67}.

Another interesting example (\cite{RR02}, page 9) is when $\varphi(t)=\log(1+t)$ and $\psi(t)=e^{t}-1$, $t\in\bR_+$. In this case,
\begin{equation}\label{xlogx}
\Phi(t)=(1+t)\log(1+t)-t,\ \ \Psi(t)=e^{t}-t-1,\ \ t\in\bR_+.
\end{equation}

To emphasize the importance of Orlicz spaces in applications, we give an illustrating example that shows  {\it the Orlicz space with the nice Young functions (\ref{xlogx}) can approximate the $\ell^2$ or $\ell^1$ space}.  Consider the following sequence of nice Young functions
$$
\Phi_k(t):=\Phi(kt)=(1+kt)\log(1+kt)-kt,\ \ t\ge0,\ k>0.
$$

Let $\Omega:=\{1,2,\dots,n\}$ and $\mu$ be the counting measure on $\Omega$. Then for each $y:=(y_1,y_2,\dots,y_n)\in\bC^n$, $\|y\|_{\Phi_k}$ equals the smallest nonnegative number $\alpha$ such that
$$
\sum_{j=1}^n \Big(1+k\frac{|y_j|}{\alpha}\Big)\log\Big(1+k\frac{|y_j|}{\alpha}\Big)-\frac{k}{\alpha}\sum_{j=1}^n |y_j| \le (1+k)\log(1+k)-k.
$$
Hence, the semi-inner-product for the corresponding Orlicz space $\cL^{\Phi_k}$ has the form
$$
[x,y]=\alpha^2\frac{\sum_{j=1}^nx_j\sgn(\overline{y_j})\log\Big(1+k\frac{|y_j|}\alpha\Big)}{\sum_{j=1}^n|y_j|\log\left(1+k\frac{|y_j|}\alpha\right)},\ \ x,y\in\bC^n.
$$
Thus, on the unit sphere where $\alpha=1$, we see that
$$
\lim_{k\to0}[x,y]=\lim_{k\to 0}\frac{\sum_{j=1}^nx_j\sgn(\overline{y_j})\log(1+k|y_j|)^{1/k}}{\sum_{j=1}^n|y_j|\log(1+k|y_j|)^{1/k}}=\frac{\sum_{j=1}^nx_j\sgn(\overline{y_j})|y_j|}{\sum_{j=1}^n|y_j|^2}=\frac{\sum_{j=1}^nx_j\overline{y_j}}{\sum_{j=1}^n|y_j|^2},
$$
which is the inner product on the $\ell^2$ space, and by the L'H\^{o}pital's rule that
$$
\lim_{k\to+\infty}[x,y]=\lim_{k\to+\infty}\frac{\sum_{j=1}^nx_j\sgn(\overline{y_j})\frac{|y_j|}{1+k|y_j|}}{\sum_{j=1}^n|y_j|\frac{|y_j|}{1+k|y_j|}}=\frac{\sum_{j=1}x_j\sgn(\overline{y_j})}{\sum_{j=1}^n|y_j|},\ \ x,y\in\bC^n,
$$
which is the semi-inner product on the $\ell^1$ space.

We describe our construction of Orlicz RKBSs as follow.

\begin{theorem} Let $\Phi,\Psi$ be a pair of nice Young functions. Choose
$$
\cW_1:=(\cL^{\Phi},\|\cdot\|_{\Phi})\mbox{ and } \cW_2:=(\cL^{\Psi}, |\cdot|_{\Psi})
$$
as a pair of feature spaces. Suppose that there exist feature maps $\Phi_1:\Omega_1\to\cW_1,\Phi_2:\Omega_2\to\cW_2$ satisfying the denseness condition (\ref{density*}). Then  $\cB_1$ defined by (\ref{B1})  is an RKBS on $\Omega_1$ with the adjoint RKBS $\cB_2$ defined by (\ref{B2})  on $\Omega_2$ endowed with the bilinear form (\ref{continuousform1}). Moreover, a reproducing kernel for $\cB_1$ is
$$
K(x,y):=\langle \Phi_1(x),\Phi_2(y)\rangle_{\cL^{\Phi}}, \ x\in\Omega_1,y\in\Omega_2.
$$
\end{theorem}

\section{Existing examples of RKBSs}

In this section, we aim at justifying that our construction in the previous section covers existing RKBSs in the literature, including the reflexive RKBS \cite{ZXZ09}, the semi-inner-product RKBS \cite{ZXZ09,ZZ11}, the RKBS constructed by Borel measures \cite{SZ11}, the RKBS with the $\ell^1$ norm \cite{SZH13}, the RKBS with positive definite functions \cite{FHY15},  and the $p$-norm RKBS \cite{XY}. In particular, $C([0,1])$ is included in our framework. Hence, we are able to write out three  reproducing kernels for $C([0,1])$.

We shall primarily use Theorem \ref{Theorem1}  to fulfill this task.  Let us keep in mind that a pair of RKBSs contain the following eleven ingredients:
$$
\Omega_1,\Omega_2,\  \cW_1,\cW_2,\ \Phi_1,\Phi_2,\ \langle\cdot,\cdot\rangle_{\cW_1\times\cW_2},\ \cB_1,\cB_2,\langle\cdot,\cdot\rangle_{\cB_1\times\cB_2}, K.
$$

\subsection{Reproducing kernel Hilbert spaces}\label{subsection31}

Our construction must include RKHSs. To show this, we shall briefly review basic theory of RKHSs \cite{Aronszajn50, Buhmann03, CZ07, Hastie09, SS02, SC08, Wendland05,ZZhao13}.

An RKHS $\cH$ on a nonempty set $X$ is a Hilbert space of certain functions on $X$ such that point evaluation functionals are continuous on $\cH$. By the Riesz representation theorem, there exists a unique reproducing kernel $K$ on $X\times X$ such that
$$
K(\cdot,x)\in\cH \mbox{ for all }x\in X, \overline{\span}\{K(\cdot,x):x\in X\}=\cH
$$
and
\begin{equation}\label{representation}
f(x)=(f,K(\cdot,x))_{\cH}\mbox{ for all }f\in\cH\mbox{ and all }x\in X
\end{equation}
where $(\cdot,\cdot)_{\cH}$ denotes the inner product on $\cH$.
Equation (\ref{representation}) is called the reproducing property in machine learning.
A function $K$ on $X\times X$ is a reproducing kernel of an RKHS if and only if there exists  a feature map $\Phi$ from $X$ to a Hilbert space $\cW$ such that
$$
K(x,y)=( \Phi(x),\Phi(y))_{\cW},\ x,y\in X.
$$
Another well-known characterization of a reproducing kernel is for it to be positive semi-definite.
There is a bijective correspondence between RKHSs and reproducing kernels. For this sake, the RKHS corresponding to a reproducing kernel $K$ is denoted by $\cH_K$.
The feature map $\Phi$ and feature space $\cW$ of a reproducing kernel may not be unique.
In particular, we say $\Phi$ is the canonical feature map of $K$ if $\Phi(x)=K(\cdot,x)$. In this case, $\cW=\cH_K$ and
$$
K(x,y)=(K(\cdot,x),K(\cdot,y))_{\cH_K}, \ x,y\in X.
$$

A reproducing kernel can be easily identified through its feature map. The following result is well-known in machine learning community \cite{XZ07,ZZhao13}.

\begin{lemma}\label{lemma31}\cite{XZ07} If $K:X\times X\to\bC$ is a reproducing kernel with a feature map $\Phi$ from $X$ to a Hilbert space $\cW$, then
$\cH_K=\{(\Phi(\cdot),u)_{\cW}:u\in\cW\}$ with the inner product
$$
\big((\Phi(\cdot),u)_{\cW}, (\Phi(\cdot),v)_{\cW}\big)_{\cH_K}=(\bP_{\Phi}v,\bP_{\Phi}u)_{\cW}, \ u,v\in\cW,
$$
where $\bP_{\Phi}$ denotes the orthogonal projection from $\cW$ onto $\overline{\span}\Phi(X)$.
If $\cW=\overline{\span}\Phi(X)$, then $\cH_K$ is isometrically isomorphic to $\cW$ through the linear mapping $(\Phi(\cdot),u)_{\cW}\to u$.
\end{lemma}

Note that the dual space of $\cH_K$ is  $\cH_K^*=\overline{\cH_K}:=\{\bar{f}:f\in\cH_K\}$. By Theorem \ref{Theorem1} and Lemma \ref{lemma31}, we obtain the following result.
\begin{example} Let $\cH_K$ be an RKHS on $X$ with the reproducing kernel $K$, and let $\Phi$ from $X$ to a Hilbert space $\cW$ be a feature map of $K$ such that $\cW=\overline{\span}\Phi(X)$. Choose
$$
\Omega_1=\Omega_2:=X,\ \cW_1:=\cW,\cW_2:=\overline{\cW},
$$
$$
\Phi_1: \Omega_1\to \cW_1, \quad \Phi_1(x):=\Phi(x), x\in \Omega_1,
$$
$$
\Phi_2:\Omega_2\to\cW_2,\quad \Phi_2(y):=\overline{\Phi(y)}, y\in \Omega_2
$$
in Theorem \ref{Theorem1}. Then
$$
\cB_1:=\Big\{f_v(x):=\langle \Phi_1(x),v\rangle_{\cW_1}=\langle \Phi(x),v\rangle_{\cW}=(\Phi(x),\bar{v})_{\cW}: v\in\cW_2, x\in\Omega_1\Big\}=\cH_K
$$
with norm $\|f_v\|_{\cB_1}:=\|\bar{v}\|_{\cW}$ is an RKBS on $\Omega_1$. Its adjoint RKBS is
$$
\cB_2:=\Big\{g_u(y):=\langle u,\Phi_2(y)\rangle_{\cW_1}=\langle u,\overline{\Phi(y)}\rangle_{\cW}=(u,\Phi(y))_{\cW}=\overline{(\Phi(y),u)_{\cW}}: u\in\cW_1, y\in\Omega_2\Big\}=\overline{\cH_K}
$$
with norm $\|g_u\|_{\cB_2}:=\|u\|_{\cW}$. The bilinear form on $\cB_1\times\cB_2$ is given by
$$
\langle f_v,g_u\rangle_{\cB_1\times\cB_2}:=\langle (\Phi(x),\bar{v})_{\cW}, \overline{(\Phi(y),u)_{\cW}}\rangle_{\cH_K}
=\big((\Phi(x),\bar{v})_{\cW}, (\Phi(y),u)_{\cW}\big)_{\cH_K} =(u,\bar{v})_{\cW}.
$$
Moreover,  $$\langle \Phi_1(x),\Phi_2(y)\rangle_{\cW_1}=\langle \Phi(x),\overline{\Phi(y)}\rangle_{\cW}=(\Phi(x),\Phi(y))_{\cW}=K(x,y)$$ is a reproducing kernel for $\cB_1$.
\end{example}
\begin{proof} The denseness condition (\ref{density}) is satisfied. By Theorem \ref{Theorem1}, the proof is complete.
\end{proof}

Therefore, we have showed that an RKHS is an RKBS with the same reproducing kernel.

\subsection{Reflexive RKBSs}

Learning in Banach spaces has received considerable attention in the past two decades. The notion of reproducing kernel Banach spaces was introduced in machine learning in 2009, \cite{ZXZ09}.  Reflexive RKBSs constructed in \cite{ZXZ09} are the first class of RKBSs with reproducing kernels.

A normed vector space $V$ is {\it reflexive} if $(V^*)^*=V$, that is,
every continuous linear functional $T$ on $V^*$ must be of the form
$$T(v^*)=v^*(u), \ \ v^*\in V^*$$ for some $u\in V$. A Banach space is reflexive if and only if its dual is reflexive, \cite{Megginson98}.

The definition and construction of reflexive RKBSs in \cite{ZXZ09} are described below.

$\\$
{\bf Definition of Reflexive RKBSs in  \cite{ZXZ09}.}  {\it An RKBS on $X$ is a reflexive Banach space $\cB$ of functions on $X$ for which $\cB^*$ is isometrically isomorphic to a Banach space $\cB^{\#}$ of functions on $X$ and the point evaluation is continuous on both $\cB$ and $\cB^{\#}$.}

$\\$
{\bf Construction of Reflexive RKBSs in  \cite{ZXZ09}.} {\it Let $\cW$ be a reflexive Banach space. Suppose that there exists $\Phi:X\to\cW$, and $\Phi^*: X\to\cW^*$ such that $\overline{\span}\Phi(X)=\cW$, $\overline{\span}\Phi^*(X)=\cW^*$. Then
$$
\cB:=\Big\{\langle u,\Phi^*(\cdot)\rangle_{\cW}: u\in\cW\Big\}\mbox{ with norm }\|\langle u,\Phi^*(\cdot)\rangle_{\cW}\|_{\cB}:=\|u\|_{\cW}
$$
is an RKBS on $X$ with the dual space
$$
\cB^{\#}:=\Big\{\langle \Phi(\cdot),u^*\rangle_{\cW}: u^*\in\cW^*\Big\} \mbox{ endowed with norm }\|\langle \Phi(\cdot),u^*\rangle_{\cW}\|_{\cB^{\#}}:=\|u^*\|_{\cW^*}
$$
and the bilinear form
$$
\big\langle \langle u,\Phi^*(\cdot)\rangle_{\cW},\langle \Phi(\cdot),u^*\rangle_{\cW}\big\rangle_{\cB\times\cB^{\#}}:=\langle u,u^*\rangle_{\cW},\ u\in \cW, u^*\in\cW^*.
$$
Moreover, $K(x,y):=\langle \Phi(x),\Phi^*(y)\rangle_{\cW}$ is a reproducing kernel for $\cB$.}
$\\$

We show that the above construction of reflexive RKBSs falls within our framework.

\begin{example} Assume the same conditions as in the above construction of reflexive RKBSs. Choose
$$
\Omega_1=\Omega_2:=X,\  \cW_1:=\cW^*,\ \cW_2:=\cW,
$$
$$
\Phi_1:\Omega_1\to\cW_1,\quad \Phi_1(x):=\Phi^*(x), \ x\in \Omega_1,
$$
$$
\Phi_2:\Omega_2\to\cW_2,\quad \Phi_2(y):=\Phi(y),\ y\in \Omega_2\ \
$$
in Theorem \ref{Theorem1}. Then
$$
\cB_1:=\Big\{f_v(x):=\langle \Phi_1(x),v\rangle_{\cW_1}=\langle \Phi^*(x),v\rangle_{\cW^*}=\langle v,\Phi^*(x)\rangle_{\cW}: v\in\cW_2=\cW, x\in\Omega_1\Big\} =\cB
$$
with norm $\|f_v\|_{\cB_1}:=\|v\|_{\cW}$ is an RKBS on $\Omega_1$. Its adjoint RKBS is
$$
\cB_2:=\Big\{g_u(y):=\langle u,\Phi_2(y)\rangle_{\cW_1}=\langle u,\Phi(y)\rangle_{\cW^*}=\langle \Phi(y),u\rangle_{\cW}: u\in\cW_1=\cW^*, y\in\Omega_2\Big\}=\cB^{\#}
$$
with norm $\|g_u\|_{\cB_2}:=\|u\|_{\cW^*}$. The bilinear form on $\cB_1\times\cB_2$ takes the form
$$
\langle f_v,g_u\rangle_{\cB_1\times\cB_2}:=\langle u,v\rangle_{\cW_1}=\langle v,u\rangle_{\cW}.
$$
Moreover, a reproducing kernel of $\cB_1$ is
$$
\langle \Phi_1(x),\Phi_2(y)\rangle_{\cW_1}=\langle \Phi^*(x),\Phi(y)\rangle_{\cW^*}=\langle \Phi(y),\Phi^*(x)\rangle_{\cW}\mbox{ for all }x\in\Omega_1,y\in\Omega_2.
$$
\end{example}
\begin{proof}  Note that $\cW$ is a reflexive Banach space. The denseness condition (\ref{density}) is hence satisfied. By Theorem \ref{Theorem1}, the proof is complete.
\end{proof}

\subsection{Semi-inner-product RKBSs}

As mentioned in \cite{ZXZ09}, the lack of an inner product may cause arbitrariness in the properties of the associated reproducing kernel of an RKBS. To overcome this,  the authors also proposed the second class of RKBSs, namely, semi-inner-product RKBSs. Thanks to the tool of semi-inner-products, existence, uniqueness and representer theorems for the standard learning schemes in semi-inner-product RKBSs were established therein.

For a  uniformly convex and uniformly Fr\'{e}chet differentiable Banach space $\cB$, there exists a unique {\it semi-inner product} $[\cdot,\cdot]_{\cB}:\cB\times\cB\to\bC$ such that for all $f,g,h\in\cB$ and $\alpha\in\bC$,
\begin{description}
\item[(i)] $[f+g,h]_{\cB}=[f,h]_{\cB}+[g,h]_{\cB}$, \ $[\alpha f,g]_{\cB}=\alpha[f,g]_{\cB}$,
\item[(ii)]  $[f,f]_{\cB}>0$ for $f\ne0$,
\item[(iii)] (The Cauchy-Schwartz inequality) $|[f,g]_{\cB}|^2\le [f,f]_{\cB}[g,g]_{\cB}$.
\end{description}
We refer to Section 2 in \cite{ZZ12} or Section 5.5 in \cite{Megginson98}, and \cite{DL07, Giles67, Lumer61, ZXZ09} for more details on semi-inner products. By the Cauchy-Schwartz inequality, for each $g\in\cB$, $f\to[f,g]_{\cB}$ is a bounded linear functional on $\cB$, which is denoted by $g^*\in\cB^*$ and called the dual element of $g$. Following this definition, we have
$$
[f,g]_{\cB}=\langle f, g^*\rangle_{\cB}.
$$
Giles \cite{Giles67} proved that if $\cB$ is a uniformly convex and uniformly Fr\'{e}chet differentiable Banach space $\cB$, then the duality mapping $f\to f^*$ is bijective from $\cB$ to $\cB^*$. Moreover,
$$
[f^*,g^*]_{\cB^*}=[g,f]_{\cB}\mbox{ for all }f,g\in\cB
$$
is a semi-inner product on $\cB^*$.

The following definition of the semi-inner-product RKBS (s.i.p. RKBS) comes from \cite{ZXZ09}.

$\\$
{\bf Definition of s.i.p. RKBS in \cite{ZXZ09}.} {\it We call a uniformly convex and uniformly Fr\'{e}chet differentiable Banach space of functions on X an s.i.p. RKBS.}

Note that a uniformly convex Banach space is reflexive. It follows that an s.i.p. RKBS is also a reflexive RKBS.
We are able to construct the s.i.p. RKBS in our framework.
\begin{example}\label{Examplereflexive} Let $\cW$ be a uniformly convex and uniformly Fr\'{e}chet differentiable Banach space, and $\Phi:X\to\cW$ and $\Phi^*:X\to\cW^*$ be such that $\overline{\span}\Phi(X)=\cW$ and $\overline{\span}\Phi^*(X)=\cW^*$. Choose
$$
\Omega_1=\Omega_2:=X,\  \cW_1:=\cW^*,\ \cW_2:=\cW,
$$
$$
\Phi_1:\Omega_1\to\cW_1,\quad \Phi_1(x):=\Phi^*(x),\  x\in \Omega_1,
$$
$$
\Phi_2:\Omega_2\to\cW_2,\quad \Phi_2(y):=\Phi(y),\  y\in \Omega_2\
$$
in Theorem \ref{Theorem1}. Then
$$
\cB_1:=\Big\{f_v(x):=\langle \Phi_1(x),v\rangle_{\cW_1}=\langle \Phi^*(x),v\rangle_{\cW^*}=[\Phi^*(x),v^*]_{\cW^*}=[v,\Phi(x)]_{\cW}: v\in\cW_2, x\in\Omega_1\Big\}
$$
with norm $\|f_v\|_{\cB_1}:=\|v\|_{\cW}$ is an RKBS on $\Omega_1$. Its adjoint RKBS is
$$
\cB_2:=\Big\{g_u(y):=\langle u,\Phi_2(y)\rangle_{\cW_1}=\langle u,\Phi(y)\rangle_{\cW^*}=[u,\Phi(y)^*]_{\cW^*}=[\Phi(y),u^*]_{\cW}: u\in\cW_1, y\in\Omega_2\Big\}
$$
with norm $\|g_{u^*}\|_{\cB_2}:=\|u^*\|_{\cW^*}$.  The bilinear form on $\cB_1\times\cB_2$ is given by
$$
\langle f_v,g_u\rangle_{\cB_1\times\cB_2}:=\langle u,v\rangle_{\cW_1}=\langle v,u\rangle_{\cW}=[v,u^*]_{\cW}.
$$
Moreover, a reproducing kernel for $\cB_1$ is
$$
\langle \Phi_1(x),\Phi_2(y)\rangle_{\cW_1}=\langle \Phi^*(x),\Phi(y)\rangle_{\cW^*}= [\Phi^*(x),\Phi^*(y)]_{\cW^*}=[\Phi(y),\Phi(x)]_{\cW}
\mbox{ for all }x\in\Omega_1,y\in\Omega_2.
$$
\end{example}
\begin{proof}  By assumptions, $\cW$ is a reflexive Banach space. The denseness condition (\ref{density}) is hence satisfied. By Theorem \ref{Theorem1}, the proof is complete.
\end{proof}

Here, we shall point out in Example \ref{Examplereflexive} that
$$
\cB_2=\Big\{[\Phi(y),u^*]_{\cW}: u\in\cW_1=\cW^*, y\in\Omega_2\Big\}=\Big\{[\Phi(y),u]_{\cW}: u\in \cW, y\in X\Big\}.
$$
We remark that $\cB_1$ and $\cB_2$ in Example \ref{Examplereflexive} are exactly the $\cB$ and $\cB^*$ constructed in (Theorem 10, \cite{ZXZ09}). They also have the same reproducing kernel.

\subsection{RKBSs by Borel measures}

In order to improve the learning rate estimate of the $\ell_1$-regularized least square regression,
a class of RKBSs by Borel measures was constructed in \cite{SZ11} to have the $\ell_1$ norm and satisfy the linear representer theorem.

Suppose that  $X$ is a locally compact Hausdorff space and denote by $C_0(X)$ the Banach space of continuous functions $f:X\to\bC$ such that for all $\varepsilon>0$, the set $\{x\in X: |f(x)|\ge\varepsilon\}$ is compact. Its dual space is isometrically isomorphic to the space $\cM(X)$ of all the regular complex-valued Borel measures on $X$ (see, Theorem 6.19 in \cite{Rudin87}). The norm of each measure $v\in\cM(X)$ is its total variation $\|v\|_{TV}$.

We are able to construct the RKBSs  by Borel measures in our framework.
\begin{example} Let $X$ be a locally compact Hausdorff space, and let $K: X\times X\to\bC$  be a continuous function such that $\overline{\span}\{K(\cdot,x):x\in X\}=C_0(X)$.
Choose
$$
\Omega_1=\Omega_2:=X,\ \cW_1:=C_0(X),\ \cW_2=\cM(X),
$$
$$
\Phi_1:\Omega_1\to\cW_1,\quad \Phi_1(x):=K(\cdot,x), \ x\in \Omega_1,
$$
$$
\Phi_2:\Omega_2\to\cW_2,\quad \Phi_2(y)=\delta_y,\  y\in \Omega_2
$$
in Theorem \ref{Theorem1}.  Then
$$
\cB_1:=\Big\{f_v(x):=\langle \Phi_1(x),v\rangle_{\cW_1}=\int_{X} K(t,x)dv(t): v\in\cW_2,x\in\Omega_1\Big\}
$$
with norm $\|f_v\|_{\cB_1}:=\|v\|_{TV}$  is an RKBS on $\Omega_1$. Its adjoint RKBS is
$$
\cB_2:=\Big\{g_u(y):=\langle u,\Phi_2(y)\rangle_{\cW_1}=u(y): u\in \cW_1,y\in\Omega_2\Big\}=C_0(X)
$$
with norm $\|g_u\|_{\cB_2}:=\sup_{y\in \Omega_2}|u(y)|$. The bilinear form on $\cB_1\times\cB_2$ is given by
$$
\langle f_v,g_u\rangle_{\cB_1}:=\langle u,v\rangle_{\cW_1}=\int_{X}u(t)dv(t).
$$
Moreover, $K$ is a reproducing kernel for $\cB_1$.
\end{example}
\begin{proof} As $\Phi_2(X)=\{\delta_y:y\in X\}$ is dense in $\cM(X)=(C_0(X))^*$ under the weak* topology, the denseness condition (\ref{densityweak}) holds true. By Theorem \ref{Theorem1}, the proof is complete.
\end{proof}

\subsection{RKBSs with the $\ell^1$ norm}\label{Subsection3.5}

The RKBS with the $\ell^1$ norm was developed in \cite{SZH13} for sparse learning.  The construction starts directly with a kernel function satisfying the following two requirements:
\begin{description}
\item[(i)] $K$ is bounded,
\item[(ii)] for all pairwise distinct sampling points $x_j\in X$, $j\in\bN$ and $c=(c_j:j\in\bN)\in\ell^1(\bN)$, $\sum_{j\in\bN}c_jK(x_j,x)=0$ for all $x\in X$ implies $c=0$.
\end{description}

Denote for any nonempty set $X$ by
$$
\ell^1(X):=\Big\{u:=(u_x\in\bC:x\in X): \|u\|_{\ell^1(X)}:=\sum_{x\in X}|u_x|<+\infty \Big\}
$$
the Banach space of functions on $X$ that is integrable with respect to the counting measure on $X$. Note that $X$ might be uncountable but for any $u\in \ell^1(X)$, $\supp u:=\{x\in X: u_x\ne0\}$ must be at most countable.  Note that $\ell^1(X)$ can be imbedded into $\cM(X)$.

$\\$
{\bf Construction of RKBSs with the $\ell^1$ Norm in \cite{SZH13}.} {\it Let $K:X\times X\to\bC$ be a kernel satisfying aforementioned  two requirements (i) and (ii). Then
$$\cB:=\Big\{f_u:=\sum_{x\in\supp u} u_xK(x,\cdot): u\in\ell^1(X)\Big\}$$
with norm $\|f_u\|_{\cB}:=\|u\|_{\ell^1}$, and $\cB^{\#}$, the completion of the vector space of functions $\sum_{j=1}^n v_j K(\cdot,y_j)$, $y_j\in X$ under the supremum norm
$$
\Big\|\sum_{j=1}^n v_j K(\cdot,y_j)\Big\|_{\cB^{\#}}:=\sup_{x\in X}\Big|\sum_{j=1}^n v_j K(x, y_j)\Big|,
$$
are both Banach space of functions on $X$ where point evaluations are continuous linear functionals. In addition, the bilinear form
$$
\Big\langle \sum_{j=1}^n  u_jK(x_j,\cdot),\sum_{k=1}^m v_k K(\cdot,y_k)\Big\rangle_{\cB\times\cB^{\#}}:=\sum_{j=1}^n\sum_{k=1}^m u_jv_k K(x_j,y_k),\ x_j,y_k\in X,
$$
can be extended to $\cB\times\cB^{\#}$ such that
$$
|\langle f,g\rangle_{\cB\times\cB^{\#}}|\le \|f\|_{\cB}\|g\|_{\cB^{\#}}\mbox{ for all }f\in\cB,\ g\in\cB^{\#}
$$
and
$$
\langle f,K(\cdot,y)\rangle_{\cB\times\cB^{\#}}=f(y),\  \langle K(x,\cdot),g\rangle_{\cB\times\cB^{\#}}=g(x)\mbox{ for all }x,y\in X, f\in\cB,\ g\in\cB^{\#}.
$$
}

We show below that RKBSs with the $\ell^1$ norm fall into our framework.

\begin{example} Let $X$ be a locally compact Hausdorff space, and let $K:X\times X\to\bC$ be bounded and continuous such that $C_0(X)=\overline{\span}\{K(\cdot,x):x\in X\}$.  Choose
$$
\Omega_1=\Omega_2:=X,\ \cW_1:=C_0(X),\ \cW_2:=\ell^1(X),
$$
$$
\Phi_1:\Omega_1\to\cW_1,\quad \Phi_1(x):=K(\cdot,x), \ \ x\in \Omega_1,
$$
$$
\Phi_2:\Omega_2\to\cW_2,\quad \Phi_2(y)=\delta_y, \ \ y\in \Omega_2
$$
in Theorem \ref{Theorem1}.  Then
$$
\cB_1:=\Big\{f_v(x):=\langle \Phi_1(x),v\rangle_{\cW_1}=\sum_{t\in\supp v} v_t K(t,x): v\in\cW_2=\ell^1(X),x\in\Omega_1\Big\}
$$
with norm $\|f_v\|_{\cB_1}:=\|v\|_{\ell^1(X)}$  is an RKBS on $\Omega_1$. Its adjoint RKBS is
$$
\cB_2:=\Big\{g_u(y):=\langle u,\Phi_2(y)\rangle_{\cW_1}=u(y): u\in \cW_1,y\in\Omega_2\Big\}=C_0(X)
$$
with norm $\|g_u\|_{\cB_2}:=\sup_{y\in X}|u(y)|$. The bilinear form on $\cB_1\times\cB_2$ is given by
$$
\langle f_v,g_u\rangle_{\cB_1\times\cB_2}:=\langle u,v\rangle_{\cW_1}:=\sum_{t\in\supp v}v_tu(t).
$$
Moreover, $K$ is a reproducing kernel for $\cB_1$.
\end{example}
\begin{proof}  The boundedness of $K$ guarantees that $f_v$ in $\cB_1$ is well-defined. The set $\Phi_2(X)=\{\delta_y:y\in X\}$ is dense in $\cM(X)=(C_0(X))^*$ under the weak* topology. The denseness condition (\ref{densityweak}) is hence satisfied. By Theorem \ref{Theorem1}, the proof is complete.
\end{proof}

In the rest of this subsection, we discuss the particular and interesting space $C([0,1])$. We shall show that $C([0,1])$ is an RKBS in our framework and shall present several explicit reproducing kernels for the space. Non-uniqueness of the reproducing kernel is caused by existence of many feature maps $\Phi_1:[0,1]\to C([0,1])$ satisfying the denseness condition $\overline{\span}{\Phi_1(X)}=C([0,1])$.

\begin{example}\label{example35} Choose
$$
\Omega_1=\Omega_2:=[0,1],\  \cW_1:=C([0,1]), \ \cW_2:=\ell^1([0,1]),
$$
$$
\Phi_1:\Omega_1\to \cW_1, \quad \Phi_1(x)(t)=1-|t-x|, \ x\in\Omega_1,\ t\in[0,1],
$$
$$
\Phi_2:\Omega_2\to\cW_2, \quad \Phi_2(y)=\delta_y, \ y\in\Omega_2
$$
in Theorem \ref{Theorem1}.  Then
$$
\cB_1:=\Big\{f_v(x):=\langle \Phi_1(x),v\rangle_{\cW_1}=\sum_{t\in\supp v} v_t(1-|t-x|): v\in\cW_2,x\in\Omega_1\Big\}
$$
with norm $\|f_v\|_{\cB_1}:=\|v\|_{TV}$  is an RKBS on $\Omega_1$. Its adjoint RKBS is
$$
\cB_2:=\Big\{g_u(y):=\langle u,\Phi_2(y)\rangle_{\cW_1}=u(y): u\in \cW_1,y\in\Omega_2\Big\}=C([0,1])
$$
with norm $\|g_u\|_{\cB_2}:=\sup_{y\in [0,1]}|u(y)|$. The bilinear form on $\cB_1\times\cB_2$ is given by
$$
\langle f_v,g_u\rangle_{\cB_1\times\cB_2}:=\langle u,v\rangle_{\cW_1}:=\int_0^1 u(t)dv(t).
$$
Moreover,  $K(x,y):=\langle \Phi_1(x),\Phi_2(y)\rangle_{\cW_1}=1-|x-y|$, $x,y\in[0,1]$ is a reproducing kernel for $\cB_1$.
\end{example}
\begin{proof} Note that every continuous function on $[0,1]$ can be approximated uniformly by piecewise linear functions.
Thus, we have $\cW_1=\overline{\span}\Phi_1([0,1])=\overline{\span}\{1-|\cdot-x|: x\in [0,1]\}$.
The denseness condition (\ref{densityweak}) is hence satisfied. By Theorem \ref{Theorem1}, the proof is complete.
\end{proof}

We can also take $\Phi_1(x)(t):=(1+t)^x$ or $\Phi_1(x)(t):=e^{tx}$ in Example \ref{example35} and they all satisfy the denseness condition (\ref{density}). This is true by the fact in complex analysis that zeros of a nontrivial holomorphic function are isolated. Correspondingly,
$$
K(x,y)=(1+y)^x \mbox{ or }K(x,y)=e^{xy} \mbox{ for all }x,y\in[0,1]
$$
can be viewed as reproducing kernels for $\cB_1$ as well.  By Definition \ref{definition2}, $\widetilde{K}(x,y)=K(y,x)$ is a reproducing kernel of  $\cB_2=C([0,1])$. Thus, we obtain three reproducing kernels for $C([0,1])$:
$$
1-|x-y|,\  (1+x)^{y},\ e^{xy},\ x,y\in [0,1].
$$

\subsection{RKBSs with positive definite functions}

The relationship between generalized Sobolev spaces and RKHSs was established by developing a connection between Green functions and reproducing kernels.
Motivated by this, the authors in \cite{FHY15} used Fourier transform techniques to construct RKBSs with positive definite functions.
Furthermore, Ye in \cite{Ye2013} developed numerical algorithms for support vector machines in those RKBSs.

Let $\phi\in L^1(\bR^d)\cap C(\bR^d)$ be positive definite, that is, for all finite distinct points $x_1,x_2,\ldots,x_n\in\bR^d$, the matrix $[\phi(x_j-x_k):j,k=1,2,\dots, n]$ is strictly positive-definite. It has been known that $\phi$ is positive definite if and only if it is bounded, its Fourier transform $\hat{\phi}$ is nonnegative, and
$$
S_{\hat{\phi}}:=\big\{\xi\in\bR^d: \hat{\phi}(\xi)\ne0\big\}
$$ has positive Lebesgue measure (see, for instance, Section 6.2 in \cite{Wendland05}). In this paper, we use the following form of the Fourier transform
$$
\hat{f}(\xi)=\int_{\bR^d}f(x)e^{-i2\pi x\cdot \xi}dx,\ \ \xi\in\bR^d,\ f\in L^1(\bR^d)
$$
and the inverse Fourier transform
$$
\check{f}(\xi)=\int_{\bR^d}f(x)e^{i2\pi x\cdot \xi}dx,\ \ \xi\in\bR^d,\ f\in L^1(\bR^d).
$$

We show below that  RKBSs (Theorem 4.1 in  \cite{FHY15} or Theorem 1 in \cite{Ye2013}) with positive definite functions fall into our framework.
\begin{example}\label{example313} Let $1<q\le 2\le p<+\infty$ and $1/p+1/q=1$. Suppose that $\phi\in L^1(\bR^d)\cap C(\bR^d)$ is a positive definite, even function on $\bR^d$ and that $\hat{\phi}^{q-1}\in L^1(\bR^d)$.  Choose
$$
\Omega_1=\Omega_2:=\bR^d,
$$
$$
 \cW_1:=\Big\{u\in L^q(\bR^d)\cap C(\bR^d): S_{\check{u}}\subseteq S_{\hat{\phi}}, {\check{u}}/\hat{\phi}^{1/p}\in L^p(\bR^d)\Big\}\mbox{ with norm }\|u\|_{\cW_1}=\Big(\int_{\bR^d}\frac{|\check{u}(\xi)|^p}{\hat{\phi}(\xi)}d\xi\Big)^{1/p}
$$
$$
\cW_2:=\Big\{v\in L^p(\bR^d)\cap C(\bR^d): S_{\hat{v}}\subseteq S_{\hat{\phi}}, \hat{v}/\hat{\phi}^{1/q}\in L^q(\bR^d) \Big\}\mbox{ with norm }\|v\|_{\cW_2}=\Big(\int_{\bR^d}\frac{|\hat{v}(\xi)|^q}{\hat{\phi}(\xi)}d\xi\Big)^{1/q}
$$
$$
\Phi_1:\Omega_1\to \cW_1,\quad \Phi_1(x):=\phi(\cdot-x), \ x\in\Omega_1,
$$
$$
\Phi_2:\Omega_2\to\cW_2,\quad \Phi_2(y):=\phi(\cdot-y),\ y\in\Omega_2
$$
and define the following continuous bilinear form on $\cW_1\times\cW_2$
$$
\langle u,v\rangle_{\cW_1\times\cW_2}:=\int_{\bR^d}{\frac{\check{u}(\xi)\hat{v}(\xi)}{\hat{\phi}(\xi)}}d\xi,\ \ u\in\cW_1,v\in\cW_2
$$
in Theorem \ref{Theorem1}. Then
$$
\cB_1:=\Big\{f_v(x):=\langle \Phi_1(x),v\rangle_{\cW_1\times\cW_2}=\int_{\bR^d}\frac{\check{\phi}(\xi)e^{i2\pi x\cdot\xi}\hat{v}(\xi)}{\hat{\phi}(\xi)}d\xi=\int_{\bR^d}\frac{\hat{\phi}(\xi)e^{i2\pi x\cdot\xi}\hat{v}(\xi)}{\hat{\phi}(\xi)}d\xi=v(x):v\in\cW_2\Big\}
$$
with norm $\|f_v\|_{\cB_1}:=\|v\|_{\cW_2}$ is an RKBS on $\Omega_1$. Its adjoint RKBS is
$$
\cB_2:=\Big\{g_u(y):=\langle u,\Phi_2(y)\rangle_{\cW_1\times\cW_2}=\int_{\bR^d}\frac{\check{u}(\xi)\hat{\phi}(\xi)e^{-i2\pi y\cdot\xi}}{\hat{\phi}(\xi)}d\xi=u(y):u\in\cW_1\Big\}
$$
with norm $\|g_u\|_{\cB_2}:=\|u\|_{\cW_1}$. The bilinear form on $\cB_1\times\cB_2$ is defined by
\begin{equation}\label{eq36}
\langle f_v,g_u\rangle_{\cB_1\times\cB_2}:=\langle u,v\rangle_{\cW_1\times\cW_2}.
\end{equation}
Moreover, $\phi(x-y)$ is a reproducing kernel for $\cB_1$.
\end{example}
\begin{proof}  To begin with, we shall check the denseness condition (\ref{density*}). Since $\phi\in L^1(\bR^d)\cap C(\bR^d)$ is positive definite,  we have $\hat{\phi}\in L^1(\bR^d)\cap C(\bR^d)$ and $\hat{\phi}^{p-1}\in L^1(\bR^d)$ for $p\ge2$. Observe that $\check{\phi}=\hat{\phi}$ as $\phi$ is an even function on $\bR^d$.
For each $x\in\bR^d$, we compute
$$
\|\phi(\cdot-x)\|_{\cW_1}=\Big(\int_{\bR^d} \frac{|\hat{\phi}(\xi)e^{i2\pi x\cdot\xi}|^p}{\hat{\phi}(\xi)} d\xi\Big)^{1/p}=\int_{\bR^d}|\hat{\phi}(\xi)|^{p-1}d\xi<+\infty,
$$
which implies $\phi(\cdot-x)\in\cW_1$ for all $x\in\bR^d$.  Furthermore,  for any $v\in\cW_2$,
$$
\langle \phi(\cdot-x),v\rangle_{\cW_1\times\cW_2}=\int_{\bR^d}\frac{\hat{\phi}(\xi)e^{i2\pi x\cdot\xi}\hat{v}(\xi)}{\hat{\phi}(\xi)}d\xi=v(x)=0 \mbox{ for all }x\in\bR^d
$$
implies $v=0$. Thus, $\span \Phi_1(\Omega_1)=\span\{\phi(\cdot-x):x\in\bR^d\}$
is dense in $\cW_1$ with respect to the bilinear form $\langle \cdot,\cdot\rangle_{\cW_1\times\cW_2}$. That $\span\Phi_2(\Omega_2)$ is dense in $\cW_2$ with respect to the bilinear form can be proved in a similar manner. The denseness condition (\ref{density*}) is hence satisfied.

Let $K(x,y):=\langle \Phi_1(x),\Phi_2(y)\rangle_{\cW_1\times\cW_2}=\langle \phi(\cdot-x),\phi(\cdot-y)\rangle_{\cW_1\times\cW_2}$. By  (\ref{eq36}), we have
$$
K(x,y)=\int_{\bR^d}\frac{\hat{\phi}(\xi) e^{i2\pi x\cdot\xi} \hat{\phi}(\xi)e^{-i 2\pi y\cdot\xi}}{\hat{\phi}(\xi)}d\xi=\int_{\bR^d}\hat{\phi}(\xi)e^{i2\pi (x-y)\cdot\xi} d\xi=\phi(x-y),\ x,y\in\bR^d.
$$
The proof is complete.
\end{proof}

We  remark that the space
$$
\cB_1=\cW_2=\Big\{v\in L^p(\bR^d)\cap C(\bR^d): S_{\hat{v}}\subseteq S_{\hat{\phi}}, \hat{v}/\hat{\phi}^{1/q}\in L^q(\bR^d)\Big\}
$$
in Example \ref{example313} is $\cB_{\Phi}^p(\bR^d)$, $2\le p<+\infty$ in Theorem 1 of \cite{Ye2013}.

\subsection{$p$-norm RKBSs}

The use of semi-inner products in the construction of RKBSs has its limitations. To overcome this issue and for the sake of sparse learning, Xu and Ye \cite{XY} constructed a class of $p$-norm RKBSs via generalized Mercer kernels.

The generalized Mercer kernel used in \cite{XY} takes the following form
\begin{equation}\label{Mercer}
K(x,y)=\sum_{n\in\bZ}\phi_n(x)\psi_n(y),\ x\in\Omega_1, y\in\Omega_2
\end{equation}
where $\{\phi_n:n\in\bZ\}$ and $\{\psi_n:n\in\bZ\}$ are sequences of functions on $\Omega_1$ and $\Omega_2$, respectively. For instance, the Gaussian kernel and the Brownian bridge kernel are generalized Mercer kernels.
Denote by $c_0$ the Banach space of sequences on $\bZ$ that vanish at infinity and are endowed with the supremum norm.
We show below that $p$-norm RKBSs fall  into our framework.

\begin{example}\label{examplep}  Let $1<p<+\infty$, $1/p+1/q=1$. Choose
$$
\cW_1:=\ell^{q}, \quad \cW_2:=\ell^{p},
$$
$$
\Phi_1:\Omega_1\to\cW_1,\quad \Phi_1(x):=(\phi_n(x):n\in\bZ),
$$
$$
\Phi_2:\Omega_2\to\cW_2, \quad \Phi_2(y):=(\psi_n(y):n\in\bZ)
$$
in Theorem \ref{Theorem1} such that $\cW_1=\overline{\span}\Phi_1(\Omega_1)$ and $\cW_2=\overline{\span}\Phi_2(\Omega_2)$. Then
$$
\cB_1:=\Big\{f_v(x):=\langle \Phi_1(x),v\rangle_{\cW_1}=\sum_{n\in\bZ}v_n\phi_n(x): x\in\Omega_1,v:=(v_n:n\in\bZ)\in\cW_2\Big\}
$$
with norm $\|f_v\|_{\cB_1}:=\|v\|_{\ell^{p}}$ is an RKBS on $\Omega_1$. Its adjoint RKBS is
$$
\cB_2:=\Big\{g_u(y):=\langle u,\Phi_2(y)\rangle_{\cW_1}=\sum_{n\in\bZ}u_n\psi_n(y): y\in\Omega_2,u:=(u_n:n\in\bZ)\in\cW_1\Big\}
$$
with norm $\|g_u\|_{\cB_2}:=\|u\|_{\ell^{q}}$.  The bilinear form on $\cB_1\times\cB_2$ is defined by
$$
\langle f_v,g_u\rangle_{\cB_1\times\cB_2}:=\langle u,v\rangle_{\cW_1}=\sum_{n\in\bZ}u_nv_n.
$$
Moreover,  $K$ defined as in (\ref{Mercer}) is a reproducing kernel for $\cB_1$.
\end{example}
\begin{proof}  Note that $\ell^{p}=(\ell^{q})^*$, where $1/{p}+1/{q}=1$, $1<p<+\infty$. The denseness condition (\ref{density}) is hence satisfied. The claim follows from Theorem \ref{Theorem1}.
\end{proof}

The RKBSs $\cB_1$ and $\cB_2$ in Example \ref{examplep} are exactly $\cB_K^p(\Omega_1)$ and $\cB^q_{K'}(\Omega_2)$ defined as in (3.5) and (3.6) of \cite{XY}, respectively. Here, $K'(x,y):=K(y,x)$, $x\in\Omega_2,y\in\Omega_1$.

\begin{example}\label{example1} Choose
 $$
\cW_1=c_0, \quad \cW_2=\ell^1,
$$
$$
 \Phi_1:\Omega_1\to\cW_1, \quad \Phi_1(x):=(\phi_n(x):n\in\bZ),
$$
$$
\Phi_2:\Omega_2\to\cW_2,\quad  \Phi_2(y):=(\psi_n(y):n\in\bZ)
$$
in Theorem \ref{Theorem1} such that $\cW_1=\overline{\span}\Phi_1(\Omega_1)$ and $\cW_2=\overline{\span}\Phi_2(\Omega_2)$. Then
$$
\cB_1:=\Big\{f_v(x):=\langle \Phi_1(x),v\rangle_{\cW_1}=\sum_{n\in\bZ}v_n\phi_n(x):v:=(v_n:n\in\bZ)\in\cW_2, x\in\Omega_1\Big\}
$$
with norm $\|f_v\|_{\cB_1}:=\|v\|_{\ell^1}$ is an RKBS on $\Omega_1$. Its adjoint RKBS is
$$
\cB_2:=\Big\{g_u(y):=\langle u,\Phi_2(y)\rangle_{\cW_1}=\sum_{n\in\bZ}u_n\psi_n(y): u:=(u_n:n\in\bZ)\in\cW_1,y\in\Omega_2\Big\}
$$
with norm $\|g_u\|_{\cB_2}:=\|u\|_{c_0}$. The bilinear form on $\cB_1\times\cB_2$ is defined by
$$
\langle f_v,g_u\rangle_{\cB_1\times\cB_2}:=\langle u,v\rangle_{\cW_1}=\sum_{n\in\bZ}u_nv_n.
$$
Moreover, $K$ defined as in (\ref{Mercer}) is a reproducing kernel for $\cB_1$.
\end{example}
\begin{proof} Note that $\cW_2=\ell^1=\cW_1^*=c_0^*$.  The denseness condition (\ref{density}) is hence satisfied. By Theorem \ref{Theorem1}, the proof is complete.
\end{proof}

The RKBSs $\cB_1$ and $\cB_2$ in Example \ref{example1} are exactly $\cB_K^1(\Omega_1)$ and $\cB^\infty_{K'}(\Omega_2)$ defined as in (3.14) and (3.15) of \cite{XY}, respectively.

\section{Representer theorems for machine learning in RKBSs}

Most machine learning tasks boil down to a regularized minimization problem. When kernel methods are used, a representer theorem asserts that the minimizer is a linear combination of the kernel functions at the sampling points. This is key to the mathematical analysis of kernel methods in machine learning \cite{CS02,CZ07,SS02}. The classical representer theorem in RKHSs was first established by Kimeldorf and Wahba \cite{KW71}.  The result was generalized to non-quadratic loss functions
in \cite{Cox90}, and to general regularizers in \cite{SHS01}. Recent references \cite{SZH13,XY,Ye2013,ZXZ09,ZZ12}  developed representer theorems for
various RKBSs introduced in the previous section. The primary purpose of this section is to present a representer theorem for RKBSs constructed in our
framework, thus unifying the representer theorems in the references.

Let $\cB_1$ be an RKBS constructed as in Theorem \ref{Theorem1} via a continuous bilinear form and a pair of feature maps.  We shall
establish a representer theorem for the regularization network in $\cB_1$.  We begin the analysis with the related minimal norm interpolation in $\cB_1$.

\subsection{Minimal norm interpolation}

The problem of minimal norm interpolation is to find a function with the smallest norm in $\cB_1$ that interpolates a prescribed set of sampled data.
Let $\bN_m:=\{1,2,\dots,m\}$, $\bx:=\{x_j:j\in\bN_m\}\subseteq\Omega_1$ be a set of $m$ pairwise distinct inputs, and $\bt:=\{t_j:j\in\bN_m\}\subseteq\bC$ be the
corresponding outputs.

The minimal norm interpolation problem looks for the minimizer
\begin{equation}\label{Interpolation}
f_{\inf}:=\arg\inf_{f\in S_{\bx,\bt}}\|f\|_{\cB_1}\mbox{ where }S_{\bx,\bt}=\Big\{f\in \cB_1: f(x_j)=t_j,\ j\in\bN_m\Big\}
\end{equation}
provided that it exists and is unique. By  (\ref{B1}) and the denseness condition (\ref{density*}),  (\ref{Interpolation}) can be equivalently reformulated as
$$
f_{\inf}=\langle \Phi_1(\cdot),v_{\inf}\rangle_{\cW_1\times\cW_2}
$$
where
\begin{equation}\label{Interpolation0}
v_{\inf}:=\arg\inf_{v\in V_{\bx,\bt}}\|v\|_{\cW_2}
\end{equation}
with
\begin{equation}\label{Interpolation1}
V_{\bx,\bt}:=\Big\{v\in\cW_2: \langle \Phi_1(x_j),v\rangle_{\cW_1\times\cW_2}=t_j,\ j\in\bN_m\Big\}.
\end{equation}
In the special case when $t_j=0$ for every $1\le j\le m$,
\begin{equation}\label{Vx0}
V_{\bx,0}=\Big\{v\in\cW_2: \langle \Phi_1(x_j),v\rangle_{\cW_1\times\cW_2}=0: j\in\bN_m\Big\}=(\Phi_1(\bx))^{\vdash},
\end{equation}
where $\Phi_1(\bx):=\{\Phi_1(x_1),\Phi_1(x_2),\dots,\Phi_1(x_m)\}$.
Here, for a subset $A\subseteq \cW_1$,
$$
A^{\vdash}:=\Big\{v\in \cW_2: \langle a,v\rangle_{\cW_1\times\cW_2}=0\mbox{ for all }a\in A\Big\}.
$$
Recall that for a subset $A$ in a normed vector space $V$,
$$
A^{\perp}:=\Big\{w\in V^*: w(a)=0\mbox{ for all }a\in A\Big\}.
$$
When $\cW_2\subseteq \cW_1^*$, $A^{\vdash}\subseteq A^{\perp}$ for $A\subseteq\cW_1$.

Next, we shall explore the condition ensuring that $V_{\bx,\bt}$ is nonempty.

\begin{lemma}\label{Lemma31} The set $V_{\bx,\bt}$ defined by  (\ref{Interpolation1}) is  nonempty for any $\bt\in\bC^m$ if and only if $\{\Phi_1(x_j):j\in\bN_m\}$ is linearly independent in $\cW_1$.
\end{lemma}
\begin{proof} One sees that $V_{\bx,\bt}$ is  nonempty for any $\bt\in\bC^m$  if and only if $\span\{(f(x_j):j\in\bN_m): f\in\cB_1\}$ is dense in $\bC^m$. Note that for each $f\in\cB_1$, there exists a unique $v\in \overline{\span}\{\Phi_2(y):y\in\Omega_2\}$ such that $f(x):=f_v(x)=\langle \Phi_1(x),v\rangle_{\cW_1\times\cW_2}$, $x\in\Omega_1$.  Using the reproducing property, we have for each $(c_j:j\in\bN_m)\in\bC^m$ that
$$
\sum_{j=1}^m c_j f(x_j)=\sum_{j=1}^m c_j \langle f_v, K(\cdot,x_j)\rangle_{\cB_1\times\cB_2}=\sum_{j=1}^m c_j\langle \Phi_1(x_j),v\rangle_{\cW_1\times\cW_2}=\Big\langle \sum_{j=1}^m c_j \Phi_1(x_j),v\Big\rangle_{\cW_1\times\cW_2}.
$$
By the denseness condition (\ref{density*}), the above equation implies that $\{\Phi_1(x_j):j\in\bN_m\}$ is linearly independent in $\cW_1$ if and only if $\span\{(f(x_j):j\in\bN_m): f\in\cB_1\}$ is dense in $\bC^m$. The proof is complete.
\end{proof}

Before moving on, we need several concepts from the theory of Banach spaces (see, for instance, Sections 1.11, 5.1 and 5.4  in \cite{Megginson98}). A normed vector space $V$ is {\it strictly convex (rotund)} if $\|t f+(1-t)g\|_V<1$ whenever $\|f\|_V=\|g\|_V=1$, $f\ne g$, and $0<t<1$, and is {\it G\^{a}teaux differentiable} if for all $f,h\in V\setminus\{0\}$, $\lim_{\tau\to0}\frac{\|f+\tau h\|_V-\|f\|_V}{\tau}$ exists. For each $f\ne0$ in a G\^{a}teaux differentiable normed vector space $V$, there exists a bounded linear functional, denoted by $\cG(f)\in V^*$ and called a G\^{a}teaux derivative of $f$, such that
 $$
\langle h, \cG(f)\rangle_V=\lim_{\tau\to0}\frac{\|f+\tau h\|_V-\|f\|_V}{\tau}\mbox{ for all }h\in V.
 $$
We make a convention that $\cG(f)=0$ if $f=0$.

Reflexivity and strict convexity of a Banach space ensure existence and uniqueness of the best approximation in the space (see, Corollary 5.1.19 in \cite{Megginson98}).
\begin{lemma}\cite{Megginson98}\label{Lemma32} If $V$ is a reflexive and strictly convex Banach space, then for any nonempty closed convex subset $A\subseteq V$ and any $x\in V$ there exists a unique $x_0\in A$ such that
$$
\|x-x_0\|_{V}=\inf\Big\{\|x-a\|_{V}:a\in A\Big\}.
$$
\end{lemma}

The last lemma needed is about orthogonality in normed vector spaces (see,  page 272,  \cite{James47}). Let $V$ be a normed vector space. We say that $f\in V$ is orthogonal to $g\in V$ if $\|f+\tau g\|_{V}\ge\|f\|_{V}$ for all $\tau\in\bC$. We call $f\in V$ orthogonal to a subspace $\cN$ of $V$ if it is orthogonal to every vector in $\cN$.

 \begin{lemma}\label{Lemma33}\cite{James47} If a normed vector space $V$ is G\^{a}teaux differentiable, then $f\in V$ is orthogonal to $g\in V$ if and only if $\langle g,\cG(f)\rangle_{V}=0$.
\end{lemma}

We are now ready to develop a representer theorem for the minimal norm interpolation in RKBSs constructed in our framework.

\begin{theorem}\label{Theorem2} (Representer theorem) Assume the same assumptions as in Theorem \ref{Theorem1}. In addition, suppose that $\cW_2$ is reflexive, strictly convex and G\^{a}teaux differentiable, and the set $\{\Phi_1(x_j):j\in\bN_m\}$ is linearly independent in $\cW_1$. Then the minimal norm interpolation problem (\ref{Interpolation0}) has a unique solution $v_{\inf}\in\cW_2$  and it satisfies
\begin{equation}\label{vinf}
\cG(v_{\inf})\in  \big((\Phi_1(\bx))^{\vdash} \big)^{\perp}.
\end{equation}
\end{theorem}
\begin{proof}  By Lemma \ref{Lemma31}, linear independence of $\{\Phi_1(x_j):j\in\bN_m\}$ in $\cW_1$ implies that $V_{\bx,\bt}$ is nonempty. Clearly, $V_{\bx,\bt}$ is closed and convex in $\cW_2$. By Lemma \ref{Lemma32}, there exists a unique $v\in V_{\bx,\bt}$, denoted by $v_{\inf}$, such that
$$
\|v_{\inf}\|_{\cW_2}=\inf_{v\in V_{\bx,\bt}}\|v\|_{\cW_2}.
$$
If $v_{\inf}=0$ in $\cW_2$, then (\ref{vinf}) holds. Observe that $v_{\inf}+V_{\bx,0}=V_{\bx,\bt}$, where $V_{\bx,0}$ is defined by (\ref{Vx0}). Since $v_{\inf}$ is the minimizer for the minimal norm interpolation problem (\ref{Interpolation0}) and $v_{\inf}+v\in V_{\bx,\bt}$ for each $v\in V_{\bx,0}$, we have
$$
\|v_{\inf}+v\|_{\cW_2}\ge \|v_{\inf}\|_{\cW_2} \mbox{ for all }v\in V_{\bx,0}.
$$
By Lemma \ref{Lemma33}, $\langle v,\cG(v_{\inf})\rangle_{\cW_2}=0$ for all $v\in V_{\bx,0}$, which implies $\cG(v_{\inf})\in V_{\bx,0}^{\perp}$. The proof is complete.
\end{proof}

When $\cW_2=\cW_1^*$, the above result can be simplified.
\begin{corollary}\label{CorollaryRepresenter}  Assume the same assumptions as in Theorem \ref{Theorem2}. If, in addition, $\cW_2=\cW_1^*$, then the minimal norm interpolation problem (\ref{Interpolation0}) has a unique solution $v_{\inf}\in\cW_2$  and it satisfies
$$
\cG(v_{\inf})\in  \span\Phi_1(\bx).
$$
\end{corollary}
\begin{proof}   Since  $\cW_2=\cW_1^*$ and $\cW_2$ is reflexive, by Theorem \ref{Theorem2}, we have
$$
\cG(v_{\inf})\in  ((\Phi_1(\bx))^{\perp})^{\perp}=\span\Phi_1(\bx).
$$
The proof is complete.
\end{proof}

We shall make some comments on the assumptions of Corollary \ref{CorollaryRepresenter}. Firstly, note that the constructed spaces $\cB_1$ and $\cB_2$ are isomorphic to $\cW_2$ and $\cW_1$, respectively. Therefore, $\cW_2=\cW_1^*$ is equivalent to $\cB_2=\cB_1^*$ (in the sense of isomorphism). By the discussion in Section 3, the condition $\cW_2=\cW_1^*$ is satisfied by RKHSs, Orlicz RKBSs, reflexive RKBSs, semi-inner-product RKBSs, RKBSs with Borel measures, RKBSs with positive definite functions, and the $p$-norm RKBSs ($1<p<+\infty$), but not satisfied by RKBSs with the $\ell^1$ norm or the 1-norm RKBSs. Secondly, for $1<p<+\infty$, $\ell^p$ and $L^p$ are reflexive, strictly convex, and G\^{a}teaux differentiable. It is also well-known that $\ell^1$ is non-reflexive. As a result, the properties of reflexivity, strict convexity, and G\^{a}teaux differentiability are satisfied by the semi-inner product RKBSs, the $p$-norm RKBSs ($1<p<+\infty$), and RKBSs with positive definite functions, but are not satisfied by the RKBSs with Borel measures, RKBSs with the $\ell^1$ norm, or the 1-norm RKBSs \cite{XY}. Consequently, additional requirements have to be imposed to ensure a linear representer theorem in the latter three spaces. For instance, a uniform boundedness of the Lebesgue constant condition on the reproducing kernel was imposed in \cite{SZH13} for the RKBS with the $\ell^1$-norm.

In addition, we remark that when an RKBS reduces to an RKHS, the above results recover the classical representer theorem for minimal norm interpolation in RKHSs.

\subsection{Regularization networks}

We consider learning a function from a prescribed set of finite sampling data
$$
\bz:=\{(x_j,t_j):j\in\bN_m\}\subseteq \Omega_1\times\bC.
$$
Let $L_{\bt}:\bC^m\to \bR_+$ be a loss function that is continuous and convex. For each $f\in\cB_1$, we set
$$
\cE_{\bz,\lambda}(f):= L_{\bt}(f({\bx}))+\lambda \phi(\|f\|_{\cB_1}),
$$
where $f({\bx}):=(f(x_j):j\in\bN_m)$, and the regularization function $\phi:\bR_+\to\bR_+$ is continuous, convex and strictly increasing with $\lim_{t\to +\infty}\phi(t)=+\infty$.
A regularization network in an RKBS $\cB_1$ takes the form:
\begin{equation}\label{RN1}
\inf_{f\in\cB_1} \cE_{\bz,\lambda}(f).
\end{equation}
Note that each $f\in\cB_1$ corresponds to a unique $v\in\overline{\span}\{\Phi_2(y):y\in\Omega_2\}$ such that
$$
f(x):=f_v(x)=\langle  \Phi_1(x),v\rangle_{\cW_1\times\cW_2},\ x\in\Omega_1.
$$
Thus, (\ref{RN1}) reduces to
\begin{equation}\label{RN2}
{\bf v}_{\inf}:=\arg\inf_{v\in\cW_2}L_{\bt}\big( (\langle \Phi_1(x_j),v\rangle_{\cW_1\times\cW_2}:j\in\bN_m)\big)+\lambda\phi(\|v\|_{\cW_2}).
\end{equation}

In order to prove the existence of the minimizer for the regularization network, we need the following result (see, Proposition 6, page 75, \cite{Ekeland83}).
\begin{lemma}\label{Lemma34}\cite{Ekeland83} Let $V$ be a reflexive Banach space and $F:V\to \bR\cup\{+\infty\}$ be convex and lower semi-continuous. If there is an $M\in\bR$ such that the set $\{v\in V: F(v)\le M\}$ is  nonempty and bounded, then $F$ attains its minimum on $V$.
\end{lemma}

Next, we establish a representer theorem for the regularization network.

\begin{theorem}\label{Theorem3} (Representer theorem) Assume the same assumptions as in Theorem \ref{Theorem2}. Then the regularization network (\ref{RN1}) possesses a unique solution $f_{{\bf v}_{\inf}}$ where ${\bf v}_{\inf}\in\cW_2$ satisfies
$$
\cG({\bf v}_{\inf})\in \big((\Phi_1(\bx))^{\vdash} \big)^{\perp}.
$$
\end{theorem}
\begin{proof} We first prove the uniqueness by contradiction.  Assume that there are two different minimizers $f_1,f_2\in\cB_1$ for (\ref{RN1}). Let $f_3:=\frac12(f_1+f_2)$.
Since $\cW_2$ is reflexive, strictly convex and G\^{a}teaux differentiable,  $\cB_1$ is reflexive, strictly convex and G\^{a}teaux differentiable as well. By the strict convexity of $\cB_1$, we have
$$
\cE_{\bz,\lambda}(f_3)=L_{\bt}\Big(\frac12f_{1}(\bx)+\frac12 f_{2}(\bx)\Big)+\lambda \phi\Big(\Big\|\frac12f_1+\frac12f_2\Big\|_{\cB_1}\Big)
<L_{\bt}\Big(\frac12f_{1}(\bx)+\frac12 f_{2}(\bx)\Big)+\lambda\phi\Big(\frac{\|f_1\|_{\cB_1}}{2}+\frac{\|f_2\|_{\cB_1}}{2}\Big).
$$
By assumptions on $L_{\bt}$ and $\phi$, it follows that
$$
\cE_{\bz,\lambda}(f_3)<\frac12 L_{\bt}(f_{1}(\bx))+\frac12 L_{\bt}(f_{2}(\bx))+\frac{\lambda}{2}\phi(\|f_1\|_{\cB_1})+\frac{\lambda}{2}\phi(\|f_2\|_{\cB_1})=
\frac12\cE_{\bz,\lambda}(f_1)+\frac12\cE_{\bz,\lambda}(f_2)=\cE_{\bz,\lambda}(f_1),
$$
contradicting that $f_1$ is a minimizer.

Next, we shall show the existence. If $f\in\cB_1$ satisfies $\|f\|_{\cB_1}>\phi^{-1}(\frac{\cE_{\bz,\lambda}(0)}{\lambda})$ then
$$
\cE_{\bz,\lambda}(f)\ge \lambda\phi(\|f\|_{\cB_1})>\cE_{\bz,\lambda}(0).
$$
Thus,
$$
\inf_{f\in\cB_1}\cE_{\bz,\lambda}(f)=\inf_{f\in E}\cE_{\bz,\lambda}(f),\mbox{ where }E:=\Big\{f\in\cB_1:\ \|f\|_{\cB_1}\le \phi^{-1}\Big(\frac{\cE_{\bz,\lambda}(0)}{\lambda}\Big) \Big\}.
$$
Clearly, $E$ is nonempty and bounded in the reflexive Banach space $\cB_1$. Observe that $\cE_{\bz,\lambda}$ is convex and continuous on $\cB_1$.
By Lemma \ref{Lemma34}, $\cE_{\bz,\lambda}$ attains its minimum on $\cB_1$.

Finally, suppose that $f_v=\langle \Phi_1(\cdot),v\rangle_{\cW_1\times\cW_2}\in\cB_1$ is the minimizer for (\ref{RN1}). We set
$$
D:=\{(x_j, f_v(x_j)):j\in\bN_m\}.
$$
By Theorem \ref{Theorem2}, there exists a unique solution $v_{\inf}\in\cW_2$ for the minimal norm interpolation (\ref{Interpolation}) with the samples $D$, and it satisfies (\ref{vinf}). It follows that  $f_{v_{\inf}}=\langle \Phi_1(\cdot), v_{\inf}\rangle_{\cW_1\times\cW_2}$ interpolates the sample data $D$ and for all $v\in\cW_2$
$$
\|v_{\inf}\|_{\cW_2}\le \|v\|_{\cW_2}.
$$
Thus, $f_{v_{\inf}}(\bx)=f_v(\bx)$ and
$$
\|f_{v_{\inf}}\|_{\cB_1}=\|v_{\inf}\|_{\cW_2}\le \|v\|_{\cW_2}=\|f_{v}\|_{\cB_1}.
$$
As $\phi$ is increasing, we get
$$
\cE_{\bz,\lambda}(f_{v_{\inf}})=L_{\bt}(f_{v_{\inf}} (\bx)+\lambda \phi(\|f_{v_{\inf}}\|_{\cB_1})\le L_{\bt}(f_v (\bx))+\lambda \phi(\|f_{v}\|_{\cB_1})= \cE_{\bz,\lambda}(f_{v}).
$$
By uniqueness of minimizer, $f_{v_{\inf}}=f_v$. As a consequence, ${\bf v}_{\inf}$ defined by (\ref{RN2}) satisfies (\ref{vinf}). The proof is complete.
\end{proof}

When $\cW_2=\cW_1^*$, the above result can be simplified.
\begin{corollary}\label{CorollaryRepresenter1} Assume the same assumptions as in Theorem \ref{Theorem2}. If, in addition, $\cW_2=\cW_1^*$ then the regularization network (\ref{RN1}) possesses a unique solution $f_{{\bf v}_{\inf}}$ where ${\bf v}_{\inf}\in\cW_2$ satisfies
 $$
 \cG({\bf v}_{\inf})\in  \span \Phi_1(\bx).
 $$
\end{corollary}

The above result covers representer theorems in existing RKBSs \cite{XY,ZXZ09,ZZ12}. Especially, it covers the classical representer theorem for regularization networks in an RKHS.

\vspace{0.5cm}
{\small
\bibliographystyle{amsplain}

}

\end{document}